\documentclass{article}
\usepackage{amssymb}
\usepackage{graphicx}
\usepackage{natbib}
\usepackage{times}
\usepackage{url}
\usepackage[hidelinks]{hyperref}
\usepackage[utf8]{inputenc}
\usepackage{graphicx}
\usepackage{amsmath}
\usepackage{amsthm}
\usepackage{algorithm}
\usepackage{algorithmic}
\usepackage{booktabs}
\usepackage{authblk}

\newtheorem{lemma}{Lemma}
\newtheorem{remark}{Remark}

\usepackage{amssymb}
\title{Hierarchically Fair Federated Learning}
\date{\vspace{-5ex}}

\begin{document}
\author[1]{Jingfeng Zhang\thanks{Equal Contributions. Preprint. Work in Progress.}}
\author[1]{Cheng Li\textsuperscript{$*$}}
\author[2]{Antonio Robles-Kelly}

\author[1]{Mohan Kankanhalli}
\affil[1]{School of Computing, National University of Singpoare, Singapore}
\affil[2]{Deakin University, Australia}

\affil[1 ]{\textit {\{j-zhang, licheng, mohan\}@comp.nus.edu.sg}}	
\affil[2 ]{\textit {antonio.robles-kelly@deakin.edu.au}}

\maketitle

\begin{abstract}
When the federated learning is adopted among competitive agents with siloed datasets, agents are self-interested and participate only if they are fairly rewarded.
To encourage the application of federated learning, this paper employs a management strategy, i.e., more contributions should lead to more rewards. 
We propose a novel hierarchically fair federated learning (HFFL) framework. 
Under this framework, agents are rewarded in proportion to their pre-negotiated contribution levels.
HFFL$+$ extends this to incorporate heterogeneous models.
Theoretical analysis and empirical evaluation on several datasets confirm the efficacy of our frameworks in upholding fairness and thus facilitating federated learning in the competitive settings.
\end{abstract}

\section{Introduction}

Traditional machine learning techniques require agents (e.g., mobile devices, terminals, companies, etc.)  to upload their data to a central server.
This approach not only increases communication between agents and the central server due to the data volume but also entails privacy risks during data transfer or due to a server breach~\cite{Andress_infoSecurity_2014}. This is an important concern since data protection regulations impose constraints on sharing of sensitive data. 

\begin{figure}[tp!]
	\centering
	\includegraphics[scale=0.22]{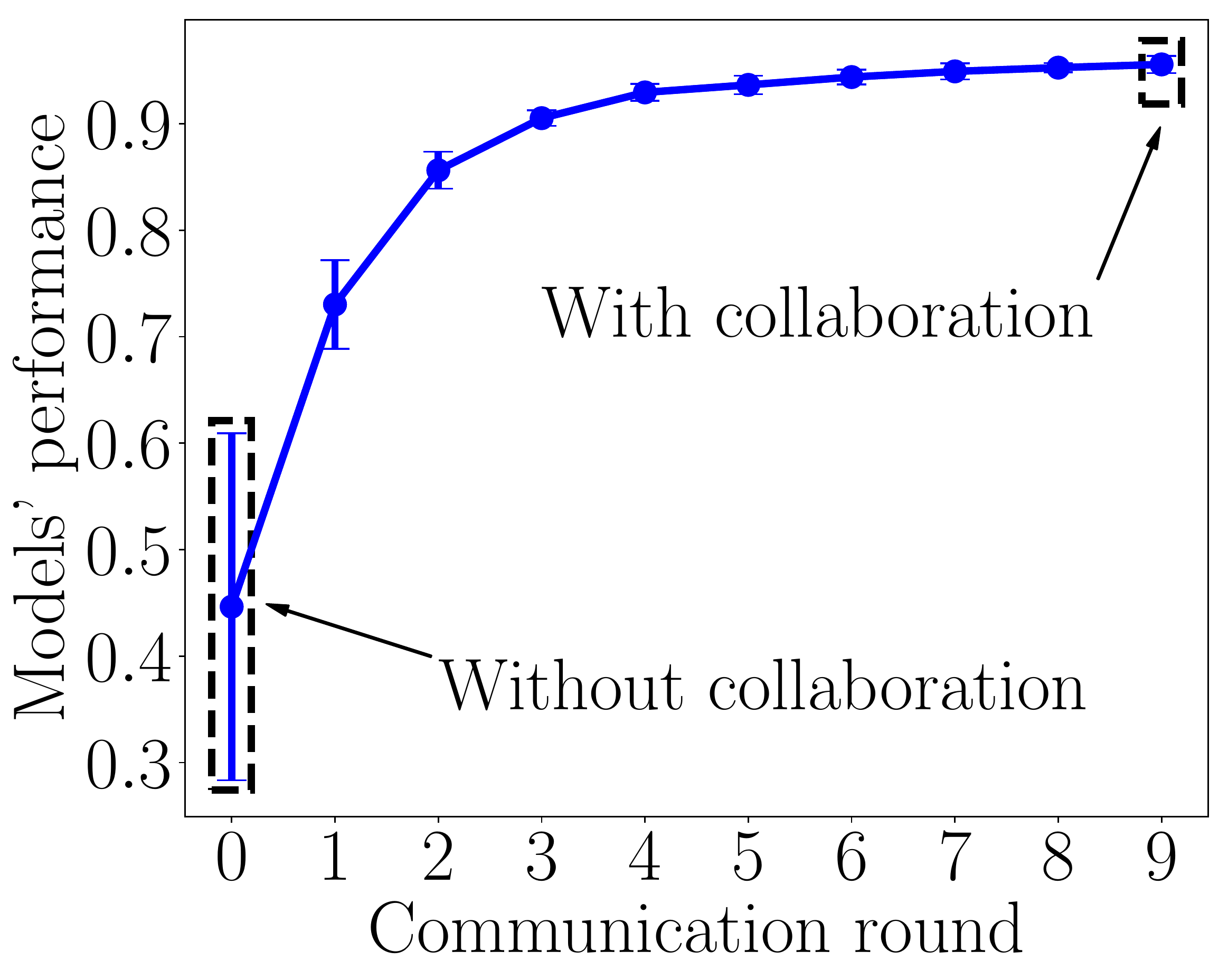}
	\includegraphics[scale=0.22]{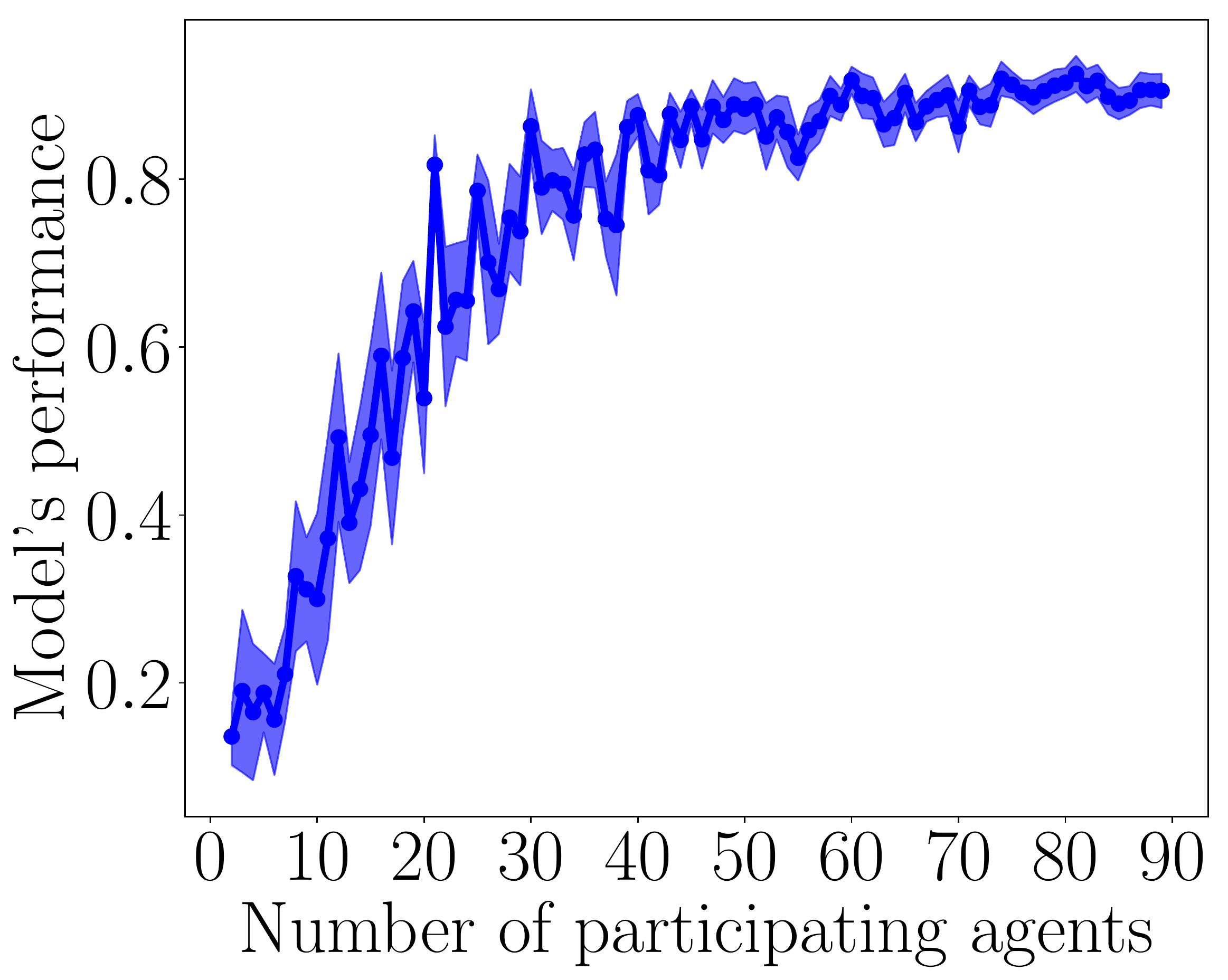}
	\vspace{-2mm}
	\caption{\textit{Left panel}: there are 30 agents (each having 1670 MNIST training data) participating in standard federated learning. At communication round 0, each agent learns a model locally. During communication rounds 1-9, they collaboratively share their data for federated learning. We report the median test accuracy and its standard deviation (error bar) over models' performance of 30 agents.
		\textit{Right panel}: there are at most 90 agents (each having 550 MNIST training data) that participate in standard federated learning. We report the median test accuracy and its standard deviation (shaded color) over participating agents models' performance. 
		%Expanded detail is in Appendix~\ref{EXP:Motivation_figure}.
	}
	\label{fig:motivation}

\end{figure}

Federated learning, a recent distributed and decentralized machine learning scheme~\cite{McMahanMRHA17_aistat_fed_learning} has attracted significant attention. In federated learning, agents maintain their data locally and collaboratively learn a global machine learning model that benefits all. Specifically, each agent sends parameters (or parameters update) of local models to the central server and receives the computed parameters of the global model from the central server. In this way, all agents can jointly train a global model without exposing their own data. This scheme has desirable properties such as privacy-preservation, efficient communication, and decentralized data storage.

Federated learning could benefits all participating agents. As we can see in the left panel of Figure~\ref{fig:motivation}, at the end of communication round 0, each agent learns a model locally. Without federated learning, each agent can only learn from their local data. A model trained on local data only has poor generalization capability. Moreover, the performance of different agents' models has a large variance. This is because without federated learning, each agent has an incomplete and biased view of the global data distribution. 
At the end of every communication round 1 through 9, all agents send their model parameter updates to the central server and then receive the global model's parameters from the central server. The global model parameters are used to improve their own model locally.  With federated learning, each agent gets significant (indirect) exposure to the global data distribution and as a result, the median test accuracy among agents increases from $45\%$ to $95\%$. Moreover, each agent model's bias with respect to the global data distribution also substantially reduces. 
The right panel of Figure~\ref{fig:motivation} illustrates the importance of the agents' participation. As more agents contribute their data and federally learn a global model, the overall performance of the global model improves. To conclude, federated learning benefits agents by improving agents' models.

The tacit assumption of federated learning is that all agents are willing to participate in the federated learning~\cite{McMahanMRHA17_aistat_fed_learning,pmlr-v97-mohri19a,Li_fair_2019}. However, in real-world situations, this assumption may not necessarily hold. Agents with siloed datasets are typically under competition and self-interested so that the agents are reluctant to collaborate unless they get some reward which is fair~\cite{kairouz2019advances,Kim_Blockchained_2019,Kang_2019_incentive,Tran_federated_2019}. This necessitates the notion of fairness in federated learning.

Previous research on fairness in federated learning typically aims to optimize model performance across all agents, either maximizing the performance of the worst agent~\cite{pmlr-v97-mohri19a} or uniforming the accuracy distribution over all agents~\cite{Li_fair_2019}. However, both these approaches do not take into account the extent of an agent's contribution in the federated learning. More specifically, they tend to protect the weak agents (e.g., agents with small amounts of data) while neglecting the strong agents (e.g., agents with large amounts of data). As a result, the strong agents find it unfair on ignorance of their greater contribution. Consequently, strong agents may abort from the federated network sacrificing the small benefit or may choose to cooperate among themselves, excluding the weak ones. In order to obtain the full benefits of federated learning (see Figure \ref{fig:motivation}), we need a fairness mechanism that is acceptable to all agents, weak or strong. 

In this paper, we employ the management of reward schemes~\cite{bratton2017human,herzberg1968one} in federated learning in the sense that the agent who \textbf{contributes more} to the federated learning should be \textbf{rewarded more}. This fairness notion is applicable in many real-world situations. An illustrative example is that a company with a large amount of data would like not to cooperate with one that has a small amount of data if both receive the same reward (e.g., learn the same machine learning model). 
It usually exerts a lot of effort and resources to collect data, which is valuable. Thus, equal reward in this case is unfair. In order to facilitate federated learning among two companies, they should be rewarded differently based on their relative contribution of data. 
This fairness notion based on proportionality also finds support from social psychology~\cite{Tornblom_1985_fairness}, which advocates that  an individual's outcomes (rewards) should match (be proportional to) his/her inputs (contributions). Similar ideas are also explored in game theory~\cite{Rabin_1993_fairness} and bandwidth allocation~\cite{Li_bandwidth_2008}.

To apply this fairness notion to federated learning, two issues have to be addressed, i.e., \textbf{(a)} how to determine the extent of an agent's contribution to federated learning and \textbf{(b)} what is the proportionate reward that the agents should receive in order to achieve the fairness. 

For the question \textbf{(a)}, data Shapley can perhaps be used to determine the extent of an agent's contribution since it is used to quantify data valuation~\cite{pmlr-v97-ghorbani19c,Wang_contributions_2019}. Specifically, the Shapley value of a datum computes the average of the marginal performance of this datum with any subsets of remaining data. The agent's contribution can then be measured by summing up the Shapley value of all data of that agent. For a pre-specified learning task, the Shapley value of an agent can be different for different chosen models -- it is model-dependent~\cite{pmlr-v97-ghorbani19c}. Thus, data Shapley is not a consistent metric to agent's contribution for a pre-specified federated learning task. The self-interested agents may quit the collaboration due to perceived unfair treatment if the reward allocation is based on Shapley value. We elaborate this in Section~\ref{section:identifying contribution of agents}. 

Instead, we propose the use of \emph{publicly verifiable factors} of agents to measure participating agent's contributions, such as the task-related data volume, data range, data collection cost, etc (details in Section~\ref{section:identifying contribution of agents}). As long as agents should reach a consensus on the chosen publicly verifiable factors, they have to agree to be contractually bound~\cite{Holmstrom_1991} and commerce the federated learning. This approach circumvents the inconsistency issue of model-dependent methods such as data Shapley and influence functions~\cite{Richardson_2019_arxiv}.

For the second question \textbf{(b)} w.r.t. fair rewards, we propose a proportionate reward system i.e. agents who are deemed more valuable will receive more model updates. 
To achieve this, we first classify all agents into different levels based on their publicly verifiable factors. Agents at the same level are deemed to have the same contribution and will be rewarded equally. Let us assume that low level agents contribute less and high level agents contribute more.
Then we train multiple models at every level such that the high-level agents only contribute the roughly same amount of data as the low-level agents own when the low-level model is federally learned. On the other hand, high-level agents get access to the low-level models to federally learn their high-level models. Our theoretical analysis shows that a model with more training data can potentially have less generalization error. In such a federated learning framework, the agents at the same contribution level share the same model and the agents at a higher contribution level can share a better model, which aligns with the proportional fairness notion. 

Our proposed method is called hierarchically fair federated learning (HFFL). 
%Note that we can train all models once without increasing time complexity due to the mini-batch training property in deep learning model and the center server just needs to know the contribution level of each agent.  %% This could bring in to algorithm analysis part. 
Based on HFFL, we also design an improvement, namely, HFFL$+$. It allows different models (e.g., different structured deep neural networks) at different levels at the expense of training time which makes the framework is more flexible and capable.
We run our algorithms on different datasets to test the fairness notion, i.e., agents in higher levels attain higher rewards.

%As a result, we define the model performance assigned by the center server as the reward of an agent.  Different from the traditional federated setting where all agents share the same global model~\cite{pmlr-v54-mcmahan17a}, we develop a novel federated learning framework where the agents in the same contribution level share the same model and the agents in the higher contribution level can enjoy a better model. To achieve it, the high level agents only contribute the roughly same volume of data as the low level agents own when the low level model is trained while the lower level agents have to donate all data they own when the high level model is trained. Note that we can train models in all levels once without increasing time complexity due to the mini-batch training property in deep learning model and the center server just needs to know the contribution level of each agent. 

%In section~\ref{}, we also justify this hierarchical federated setting can successfully satisfy the fairness we define forehead. The experiment results in Section~\ref{} on MNIST and CIFRA10 show the priority \footnote{+++(JF): superiority} of our method. 

Our main contributions in this paper are:
\begin{itemize}
	\item We employ the reward schemes of the management in federated learning, i.e., more contribution leads to more reward.
	\item We propose a novel hierarchical federated learning framework to achieve proportional fairness so that it facilitate collaborations among agents.
	\item Empirical evaluation of our methods on four datasets, i.e., census dataset ADULT, vision datasets MNIST as well as Fashion MNIST and text dataset IMDB confirms our frameworks in upholding the fairness.
\end{itemize}

The rest of this paper is organized as follows. The related work is surveyed in Section~\ref{Section:fairness_review}. How to measure an agent's contribution is discussed in Section~\ref{section:identifying contribution of agents}. The proposed federated learning framework and theoretical analysis are presented in Section~\ref{section:HFFL}. Experimental results are showed in Section~\ref{section:experiments} followed by the conclusion and future work in Section~\ref{section:conclusion}.

\section{Related work}
\label{Section:fairness_review}
\textbf{Fairness in federated learning}. Fairness in machine learning is often defined as a notion of protecting against discrimination of some specific features in data, e.g., minorities. A number of prior studies has focused on addressing this feature-level fairness in models. Two commonly adopted strategies are per-processing sensitive features such as deletion or transformation ~\cite{Feldman:2015:CRD:2783258.2783311} and modifying existing models to limit discrimination ~\cite{pmlr-v54-zafar17a,Goh_NIPS_2016}. Fairness has also been considered in resource division in multi-agent systems. In this context, the resource provided by the environment needs to fairly shared by all agents. Some typical work includes the maximin sharing policy~\cite{NIPS2014_5588}, which improves the performance of the worst agent, and the fair-efficient policy which makes the variation of agents' utilities as small as possible ~\cite{Jiang_fairnessMA_2019}.

In federated learning, existing work on fairness aims to ensure accuracy across all agents. For example, \cite{pmlr-v97-mohri19a} proposed agnostic federated learning (AFL), which minimizes the maximal loss function of all agents with a consideration on the overall performance. \cite{Li_fair_2019} proposed q-Fair Federated Learning (q-FFL) to encourage a more uniform accuracy distribution across all agents, where $q$ is a trade-off parameter between fairness and accuracy. They do not take the contribution of agents into consideration. Our notion of fairness is not to optimize the accuracy across all agents, but is to ensure that the agents which contribute more receive proportionately more reward.  \\ 
\textbf{Incentive design in federated learning}.  Our employed fairness is also a kind of an incentive to encourage participation of agents based on their resources. \cite{Richardson_2019_arxiv} used a similar incentive where the central server pays agents proportional to their data valuation. They employed influence functions for data valuation. Influence functions quantify how much the model’s predictions would change if that datum was not used in the training process and is in fact the leave-out-one (LOO) method. The LOO method assigns zero value to the duplicate of a datum~\cite{pmlr-v97-ghorbani19c}. If two agents happened to have exactly the same data, then the extra data is deemed to be of zero value. 
%If these training points have exact copy in an agent since deleting them does not affect the model's performance. 
Furthermore, the LOO method is also model-dependent, which has the same issue of data Shapley.       
%Kang \textit{et al.}~\shortcite{Kang_2019_incentive} identify the agent's contribution through observing the agent's computational resource and the model accuracy returned by the agent. However, without exposing local data, it is challenging to assess the value. In addition, their methods also have the model-dependent issue. 
%used the incentive mechanism that the more contributed computation resource leads to faster local model training, thus bringing higher rewards. The agent's contribution is identified 
%Both agents' contribution and reward can be obtained by maximizing the profit of the center server, which is a function of agents' contribution and reward. 
In contrast, our work uses publicly verifiable factors of agents to identify agents' contributions, which are free from model dependency.    

%\section{Hierarchically Fair Federated Learning}
\label{Section:incentive_for_collaboration}

%Our proposed framework rewards more to the agents who contributes more to federated learning. In this section, we first identify agents' contribution and then train hierarchical models so that the agents at the different contribution levels can receive different level utilities. 
\section{Identifying Contribution of Agents }
\label{section:identifying contribution of agents}
Identifying agents' contribution is a key driver to the success of collaboration~\cite{Parung_business_2008}. 
In this subsection, we first show why data Shapley is not a consistent metric of agents' contributions in federated learning. We then discuss economic and social factors that are used to measure different agents' contribution under the collaboration.

\subsection{Data Shapley Is Not A Suitable Metric}
For the same task, data valuation based on data Shapley is (1) model dependency, (2) negativity, and (3) evaluations metric dependency~\cite{pmlr-v97-ghorbani19c}. In this section, we verify points (1) and (2) and show why data Shapley is not a suitable measure of the contributions of the agents and hurdle their incentives to participate in the federated learning. 

\begin{figure}[tp!]
	\centering
	\includegraphics[scale=0.2]{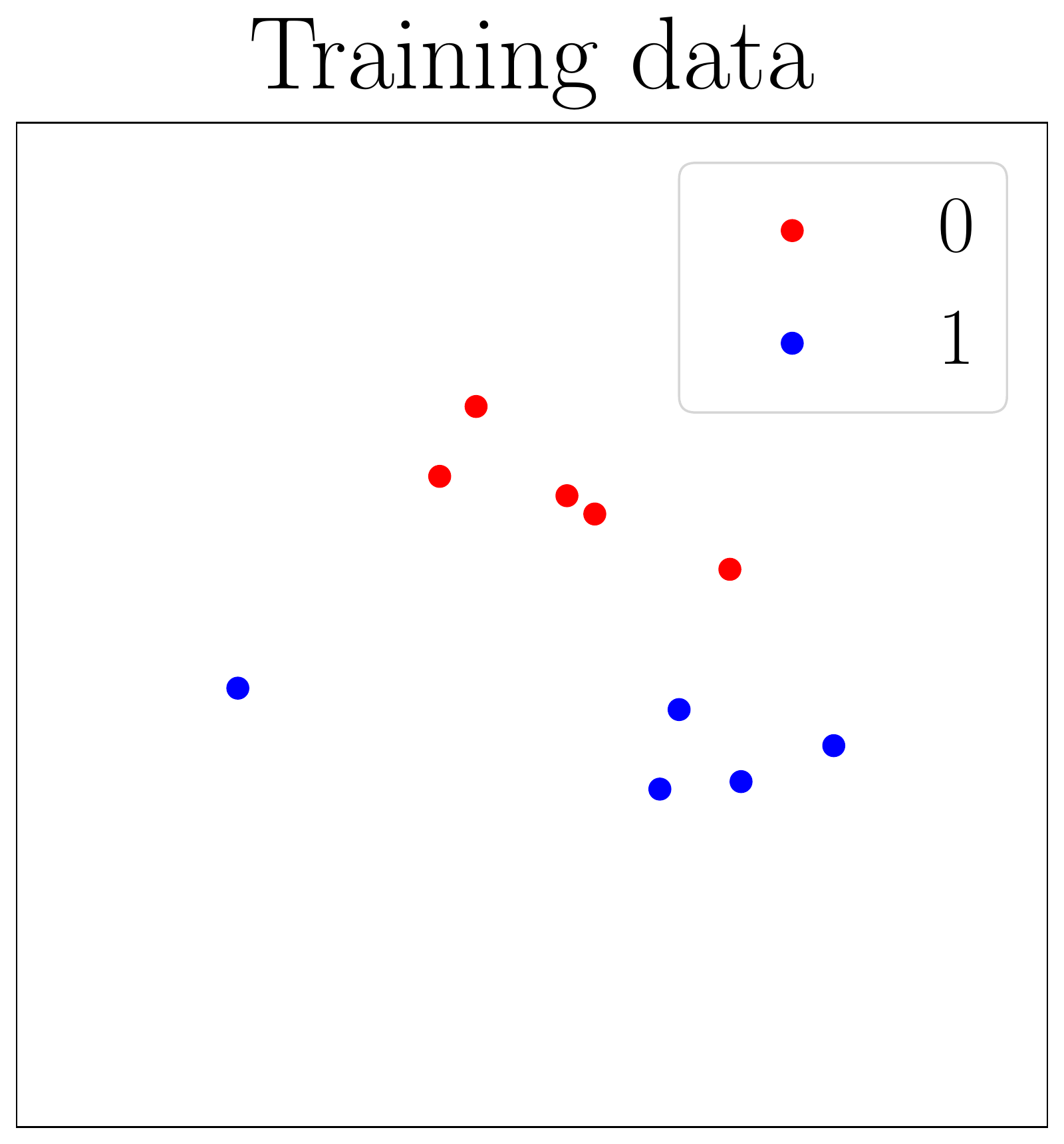} 
	\hspace{11mm}
	\includegraphics[scale=0.2]{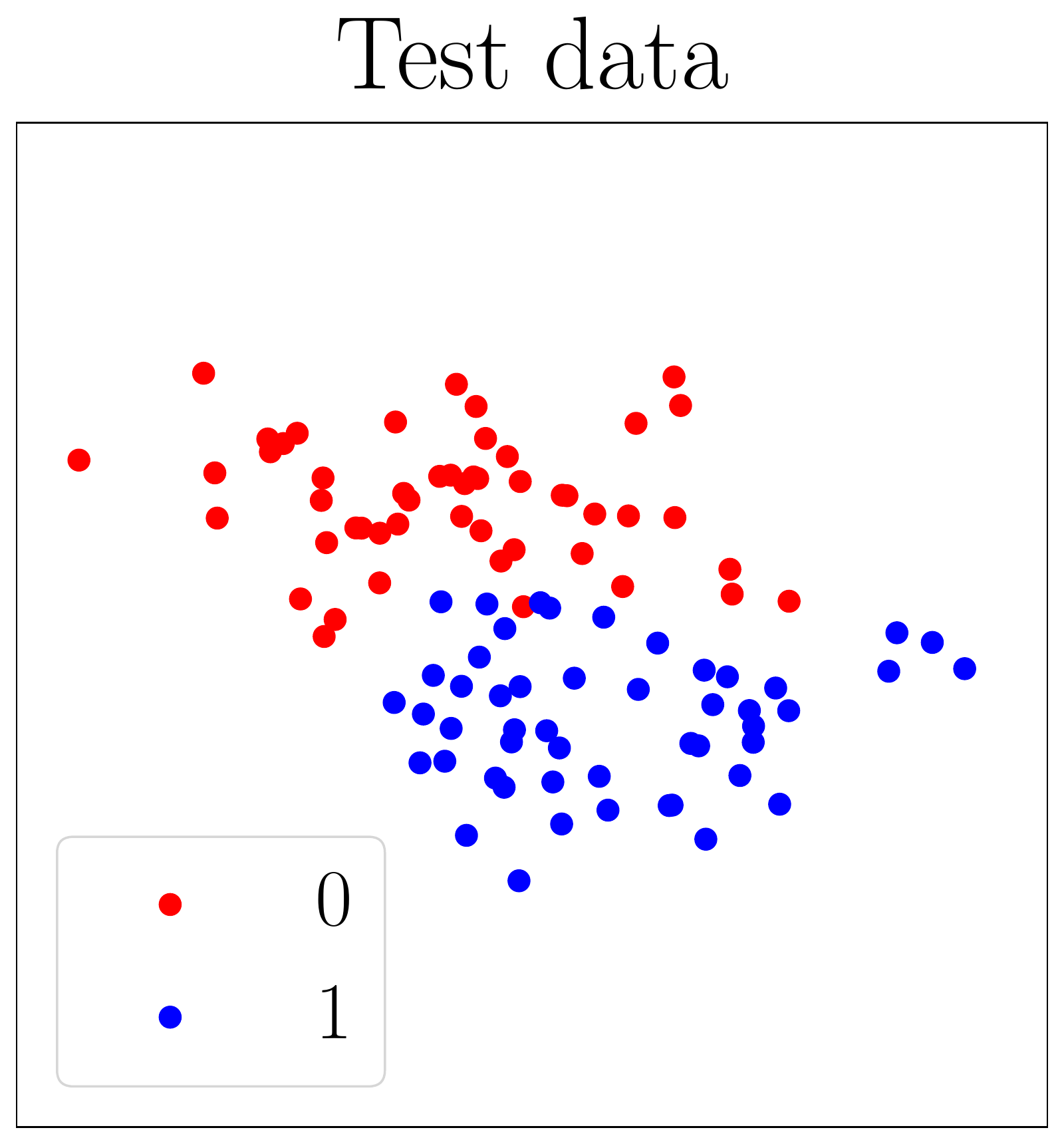} \\
	\includegraphics[scale=0.22]{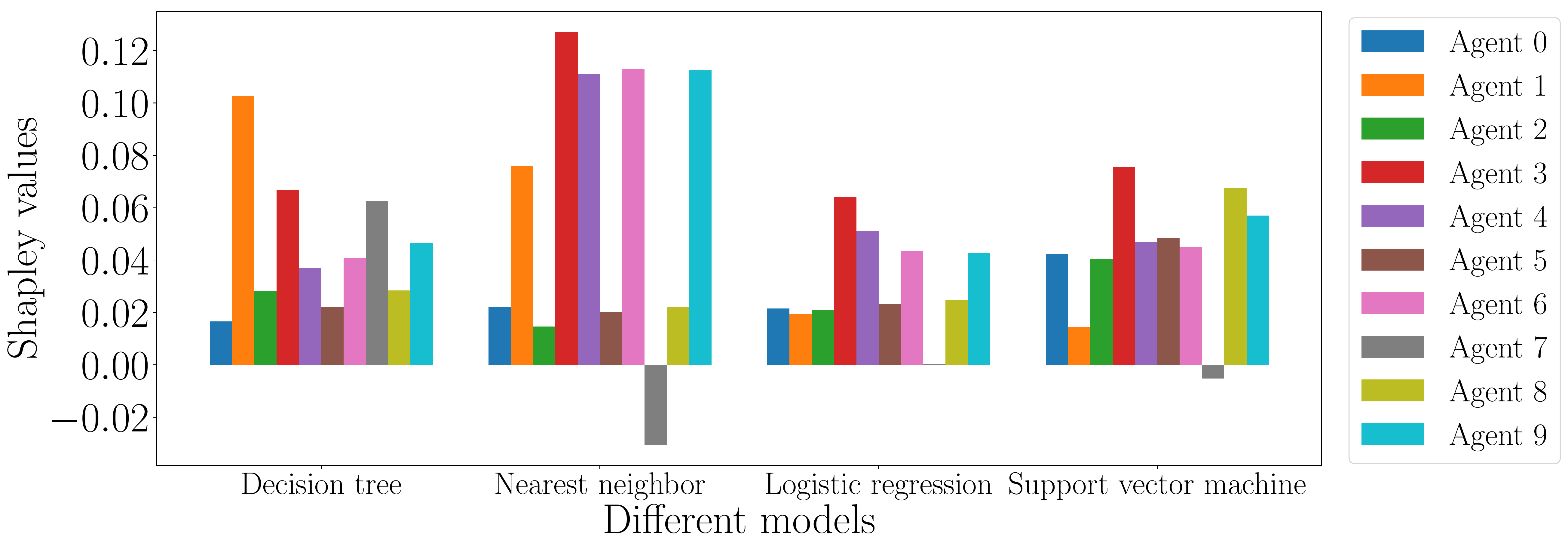}
	\caption{\textit{Top}: training data and test data for calculating data Shapley. Each agent holds two data from training and uses the same test data.
		\textit{Bottom}: Shapley values of different agents across different models for the binary classification.}
	\label{fig:Shapley_values_across_agents}
\end{figure}

In machine learning, data Shapley is used to measure the value of each training datum towards the predictor performance. An agent's contribution is measured by summing the Shapley value of all data from the agent~\cite{pmlr-v97-ghorbani19c}.

Given a learning algorithm $\mathcal{A}$ taking a training dataset $S \subseteq \mathcal{D} = \{(\mathbf{x}_i, y_i)\}^{m}_{i=1}$ as input and returning a model, data Shapley is a metric to quantify the value of each training data towards the predictor performance. 
In supervised machine learning, the Shapley value $\phi_i$ of a datum $(\mathbf{x}_i, y_i)$ is calculated by
\begin{align*}
\phi_i = C \sum_{S \subseteq D - \{(\mathbf{x}_i, y_i)\}} \frac{V(S \cup \{(\mathbf{x}_i, y_i)\}) - V(S)}{ \binom{n-1}{|S|}  }
\end{align*}
where $C$ is a scaling constant and $V$ is the performance score of the predictor trained on dataset $S$. $V(S)$ is in short for $V(\mathcal{A}(S))$.

We conduct an experiment to illustrate Shapley value of agents, where there are 10 agents, each holding only one data point at the top left panel in Figure~\ref{fig:Shapley_values_across_agents}. They collaboratively learn a binary classification model. The performance score $V(S)$ is calculated based on the accuracy of the returned predictor $\mathcal{A}(S)$ on test data in the top right panel. 
The bottom panel in Figure~\ref{fig:Shapley_values_across_agents} shows the Shapley value across different agents over different learning algorithms, i.e., decision tree, nearest neighbor, logistic regression and support vector machine. For the same classification task, we can see that the contribution of agents measured by data Shapley varies significantly over the learning algorithms. 
For example, Shapley value of Agent 4 is largest when $\mathcal{A}$ uses nearest neighbor, but it becomes relatively smaller when decision tree is used. 

In addition, agent 7 even has a negative Shapley value when nearest neighbor and support vector machine is used. It is unrealistic to expect self-interested agents to participate if they get negative rewards and pay the cost of contributing data. 
Furthermore, data with negative Shapley value might be valuable and cannot be neglected. These data points might be outliers which impact negatively on the performance of the predictor, but they may represent the rarest of cases that should be recognized and classified~\cite{mathews2019learning}. For example, they could be critical in areas such as medical diagnosis, IoT sensors, fraud detection and intrusion detection.
%Including a data point data with negative Shapley value, a learning algorithm will return a predictor with slightly degraded test performance. 

To conclude, for a specific learning task (classification or regression), self-interested agents can find reward allocation based on data Shapley to be unfair and may quit from the collaboration. Thus, data Shapley is not a suitable metric to identify an agent's contribution.

\subsection{Contribution Measures}
Exactly identifying agents' contribution is often difficult in a federated network since multiple factors are involved such as agents' reputation, communication bandwidth, computation resource, data amount and quality, etc. Inspired by the rating mechanism in the finance industry such as Standard \& Poor's which issues credit ratings with the levels from AAA to D for public and private companies, governmental entities and etc, we propose  classifying all agents into multiple distinct contribution levels and then reward them based on their contribution levels.
%which we will present in section 3.2.\footnote{+++ Shall we need to add such present? Is that mean: which we will technically manifest it? }.  

We use the publicly verifiable factors of agents to classify them, such as data quality, data volume, cost of data collection, etc. Our strategy is more likely to be task-dependent rather than model-dependent. Agents should come to a consensus on which factors are relevant for the given task. It is also practically enforceable using contracts~\cite{Holmstrom_1991}, and collaboration can be assured given fair rewards. This simple strategy is easy to understand. What is more, a trustworthy third party can also be employed to rate agents and classify them into different levels in practice.

\begin{figure}[tp!]

	\begin{centering}
		\includegraphics[scale=0.3]{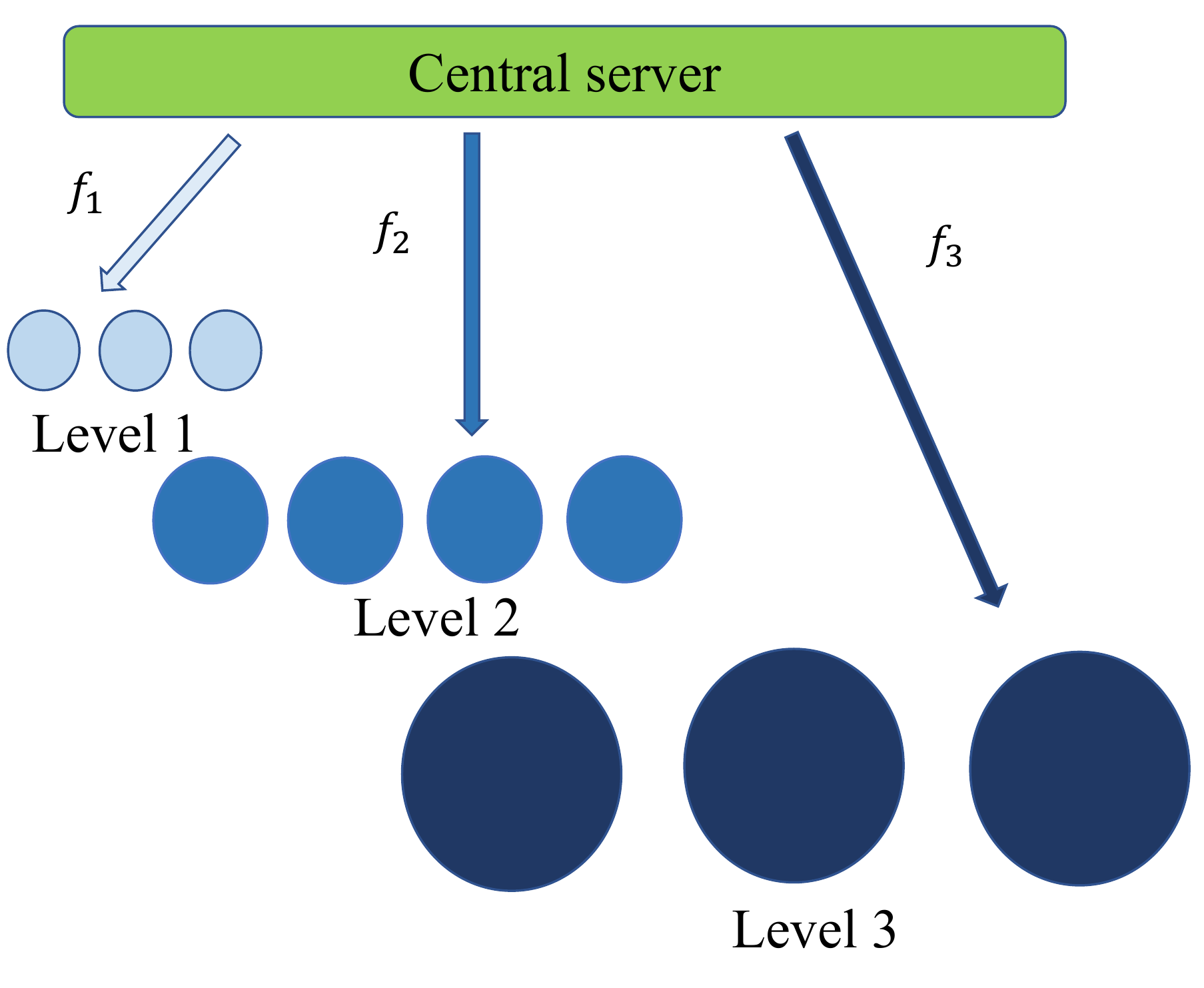}
		\includegraphics[scale=0.37]{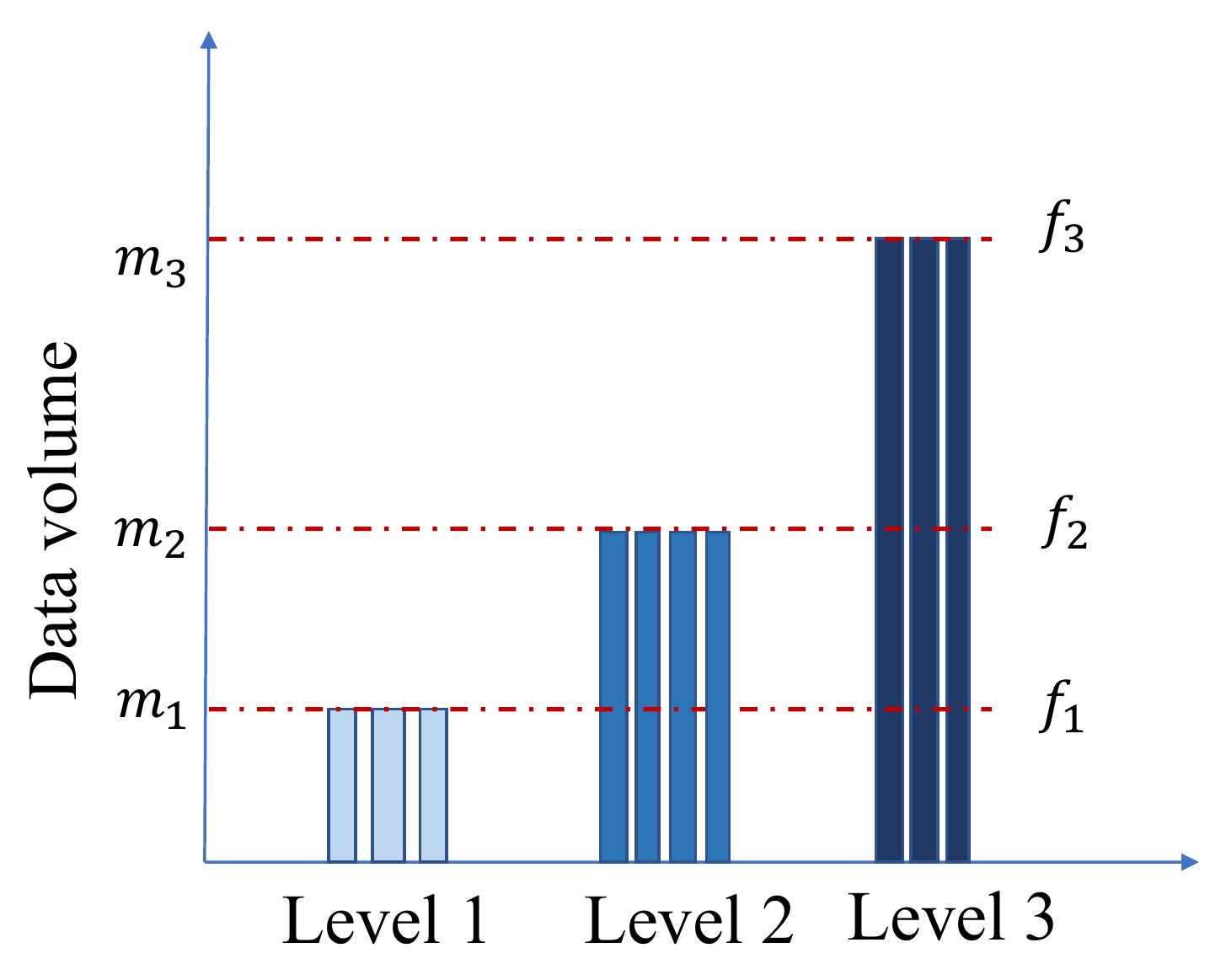}
		\par
	\end{centering}

	\caption{An illustrating example of hierarchically fair federated learning (HFFL). There are three levels ($L = 3$). The \textit{left} figure presents the hierarchical structure, where agents are categorized into different levels for collaboration. The \textit{right} figure shows how agents contribute their local data for federated learning. 
		Agents at a lower level contribute less data and get a lesser reward (e.g., obtaining model $f_1$), and agent at a higher level contribute more data and get a better reward (e.g., obtaining model $f_3$).  }
	\label{fig:HFFL_Structure}

\end{figure}

\section{Hierarchically Fair Federated Learning}
\label{section:HFFL}
In this section, we design a novel hierarchical federated learning framework, i.e., HFFL, that incorporates the fairness notion. 
Agents who contribute more to federated learning are rewarded more in this framework.  The agents at the different contribution level thus receive the different number of model updates. 

\subsection{HFFL Algorithm}
Once the participating agents come to a consensus about their contribution levels, they can collaborate using our HFFL algorithm. In HFFL, all agents hold their data locally without sharing data directly but they share the model parameters with the central server as FedAvg~\cite{McMahanMRHA17_aistat_fed_learning} does. 

In our HFFL, there are $L$ contribution levels with each level having $N_{l}~(l=1,\cdots,L)$ agents $A_{l1}, \cdots, A_{{l}N_{l}}$.  
We assume agents at a given level have the same amount of data (or have the same metric value that takes all data characteristics into account). Agents at level $l$ has $m_{l}$ amount of data. 
Agents at the same level are rewarded with the same machine learning model $f_{l}$. We provide an example of the HFFL structure in Figure~\ref{fig:HFFL_Structure}. There are 3 agents ($N_1= 3$) at level 1 with each agent holding $m_1$ amount of data locally, 4 agents ($N_2 = 4$) at level 2  with each agent holding $m_2$ amount of data, and 3 agents  ($N_3 = 3$) at level 3 with each holding $m_3$ amount of data. Agents at a lower (higher) level normally contribute less (more) amount of data, thus $m_1 < m_2 < m_3$. 

HFFL aims to federally learn different models at different contribution levels.
To learn the model $f_{l}$ at the level $l$, higher-level agents contribute the same amount of data as level-$l$ agents have, while level-$l$ and lower-level agents should contribute all the data they have.
The data of high-level agents contributing to the lower level is randomly sampled (since only a subset is used). In our HFFL, we assume agents and the central server are benign so there is no cheating. We defer the adversarial settings to future work.

For example, in Figure~\ref{fig:HFFL_Structure}, model $f_1$ is federally learned based on $ ( (N_1+N_2+N_3) \cdot m_1 )$ amount of data, which includes all data contributed by agents at level 1 and $m_1$ amount of data sampled from every agent at higher levels.  Similarly, model $f_2$ is federally learned based on $( N_1 \cdot m_1 + (N_2+N_3) \cdot m_2 )$ amount of data, which includes all data from agents at level 1 and level 2 plus $m_2$ amount of data sampled from each agent at level 3. The model $f_3$ is federally learned based on $( N_1 \cdot m_1 + N_2 \cdot m_2 + N_3 \cdot m_3 )$ amount of data.
Our theoretical analysis in Section~\ref{section:theoretical_analysis} shows that a model learned with more data has higher confidence to generalize better. 
Therefore, the agents at a higher contribution level have a better machine learning model. We assume that a better model implies a higher reward.
The implementation details are in Algorithm~\ref{alg:HFFL}.
%Figure~\ref{fig:HFFL_Structure} illustrates our HFFL with $N=3$ contribution levels. In our hierarchical structure, agents in a lower (higher) level normally contribute less (more) amount of data, thus they are rewarded a model containing information of less (more) data.

%An agent could obtain models of lower levels, but cannot get access to any models of higher levels. The detailed implementation is in Algorithm \ref{alg:HFFL}.

\begin{algorithm}[h!]
	\caption{Hierarchically fair federated learning (HFFL)}
	\label{alg:HFFL}
	\textbf{Input}: Level number $L$, each level $l$ having $N_{l}$ agents from $A_{l1}, \cdots, A_{{l}N_{l}}$, each agent at level $l$ contributes $m_l$ amount of its local data,
	machine learning models with the same architecture $\mathcal{F}$, a central server $\mathcal{C}$, communication rounds $T$.  \\ %\textbf{Parameter}: Optional list of parameters\\
	\textbf{Output}: Models $\{f_{l}\}_{l=1}^{L}$ for $L$ levels. 
	
	\begin{algorithmic}[1] %[1] enables line numbers
		\STATE $\mathcal{C}$ initializes a model $f_0 \in \mathcal{F}$.
		\FOR{level $l = 1, 2, ..., L$} 
		\STATE $\mathcal{C}$ initializes a model for level $l$: $f_{l} \leftarrow f_{l-1}$.
		\STATE Denote all participating agents at level-$l$ federated learning as an agent set $A = \{A_{ij}\}, i=l,\cdots,L, j=1,\cdots,N_{i}$. 
		\FOR{communication round $t = 1, 2, \dots, T$ }
		
		\STATE $\mathcal{C}$ sends current parameters $w_{t}^{l}$ of $f_{l}$ to each agent in set $A$.
		
		\FOR{each agent $A_{ij}$ in $A$ \textbf{in parallel}}
		\STATE $w^{ij}_t \leftarrow $ AgentUpdate$(i, j, w_{t}^{l})$ \\
		\COMMENT{* $w^{ij}_t \leftarrow \eta \nabla \ell (w_{t}^{l}; S^{l}_{ij})$, where $S^{l}_{ij}$ is agent $A_{ij}$ local data for level $l$, $|S^{l}_{ij}| = m_l$, $\eta$ is the learning rate.}
		%\COMMENT{*Agent $A_{ij}$ locally trains on its local data of $m_l$ amount and then uploads its local model parameter $w^{ij}_t$ to central server C.}
		\ENDFOR 
		\STATE $ w^{l}_{t+1} \leftarrow \frac{1}{N_{l}+\cdots+N_{L}}\sum_{i={l}}^{L}\sum_{j=1}^{N_{i}}w_{t}^{ij}$ \\
		\COMMENT{*$\mathcal{C}$ receives local model parameters $w^{ij}_t$ of each agent $A_{ij}$ in $A$ and then updates $f_{l}$ with $w^{l}_{t+1}$. }
		%\STATE $C$ synchronizes $(L_n + L_{n+1} + ,..., + L_N)$ agents with the new parameter $w$.
		\ENDFOR
		\STATE Output federated learned model $f_l$ for level $l$.
		\ENDFOR
	\end{algorithmic}
\end{algorithm}

\begin{remark} \upshape
	Each agent holds their data locally without sharing directly. Without federated learning, each agent has access to very limited (its own) data. 
	But when an agent participates in federated learning, it has (indirectly) access to other agents' data.
	Besides, in terms of knowing unknowns, it is strictly better off than when not participating. 

	In addition, an agent at a higher contribution level needs to contribute more amount of data but ends up obtaining a better model. Thus, HFFL encourages agents to collect and contribute more data in order to get promoted to a higher level.
\end{remark}

\begin{remark} \upshape
	Our HFFL framework is flexible. 
	It allows an agent at a lower level (e.g., level $l=1$) to get promoted to a higher level (e.g., level $l=2$) as long as all agents at higher levels (i.e., $l\geq2$) have no objections so that this agent could obtain a better model (i.e., $f_2$). The better model is trained on more data. But this agent will probably need to somehow compensate the higher-level agents (e.g. pay money or promise more data in future) to get approval for such a promotion.
\end{remark}

\begin{remark} \upshape
	The HFFL framework can maintain the same training time as FedAvg~\cite{McMahanMRHA17_aistat_fed_learning} since HFFL finetunes low-level models when training a higher-level model. Note that HFFL has the flexibility to incorporate other different federated learning strategies, e.g., privacy preserving federated learning~\cite{DBLP:conf/sp/NasrSH19_reza,nips/AgarwalSYKM18_dp_fl_distributed_sgd} and robust federated learning~\cite{HaoLXLY19_dp_fl}. 
\end{remark}

\paragraph{HHFL$+$} In HFFL, agents at different levels have models with the same architecture.
However, a more complicated model tends to overfit the small number of data. Agents in lower levels having less data probably prefer simpler machine learning models, while agents having more data probably prefer more complex models. To facilitate this, we design an improved version of HFFL, namely, HFFL$+$. It runs HFFL multiple times, each with different $\mathcal{F}$, i.e., models with different architectures. We then select the best-performing models for each level based on the models' test accuracy.

\subsection{Theoretical Analysis}
\label{section:theoretical_analysis}
In this section, we theoretically justify (a) an agent that participates in federated learning has gain, (b) a higher level model can potentially have a less generalization error, which is aligned with our proportional fairness notion.

Let $S= \{ (\mathbf{x}_{i},y_{i})\}_{i=1}^{m}$ ($|S| = m$) denote the training data sampled from an unknown distribution $\mathcal{D} = \mathcal{X} \times \mathcal{Y}$. Suppose we have a finite model set $\mathcal{F},$ and the bounded loss function $\ell$. 
For a specific level in HFFL, the federated learning algorithm is to utilize all available training data $S$ to learn a model $f: \mathcal{X} \rightarrow \mathcal{Y}$ ($f \in \mathcal{F}$). The aim is to minimize generalization error $L_{\mathcal{D}}(f) {:=} \mathbb{E}_{(\mathbf{x}_i,y_i)\sim\mathcal{D}}\ell(f;\mathbf{x}_i,y_i)$ through minimizing the empirical error $L_{S}(f){:=}\frac{1}{m}\sum_{i=1}^{m}\ell(f;\mathbf{x}_{i},y_{i})$. We have the following lemma. 

\begin{lemma}
	\label{lemma1}
	Given a training set $S$ with $m$ data, an error rate $\epsilon$ and a bounded loss $\ell$ within the range $[a,b]$, set $\delta=2|\mathcal{F}|\exp(-2m\epsilon^{2}/(b-a)^2)$. Then, 
	\begin{align}
	\left|L_{\mathcal{D}} (f) - L_{S}(f) \right|<\epsilon, \exists f\in\mathcal{F}   \nonumber
	\end{align}
	holds with probability $\geq1-\delta$.
\end{lemma}

\begin{proof}
	To bound the difference between the empirical error $L_{S}(f)$ and the generalization error $L_{\mathcal{D}}(f)$, %The law of large numbers states that empirical risk converges to their true expectation when the number of samples $m$ goes to infinity.  However, in our case that the number of samples is finite. 
	we use Hoeffding's inequality.\\ 
	%which directly gives bound of $\left| L_{\mathcal{D}} (f) - L_{S}(f) \right| $.
	%on the difference between the empirical mean of a sequence of random variables and its expected value.
	(\textbf{Hoeffding's inequality}) Let $\theta_{1},\cdots,\theta_{m}$ be a
	sequence of $i.i.d.$ random variables and assume that for all $i$,
	$\mathbb{E}[\theta_{i}]=\mu$ and $\mathbb{P}[a\leq\theta_{i}\leq b]=1$.
	Then, for any $\epsilon > 0$,
	\begin{align}
	\mathcal{P}\left(\left|\frac{1}{m}\sum_{i=1}^{m}\theta_{i}-\mu\right|>\epsilon\right)\leq2\exp(-2m\epsilon^{2}/(b-a)^{2}) \nonumber
	\end{align}
	The proof of Hoeffding's inequality can be found in \cite{Shai_2014_understandingML}. 
	
	%Back to our problem. Since $f$ is fixed and
	Since $\{(x_{i},y_{i})\}_{i=1}^{m}$ are $i.i.d$ sampled, the $\{\ell(f;x_{i},y_{i})\}_{i=1}^{m}$ are $i.i.d$. Let us set $\theta_{i}$ $=$ $\ell(f;x_{i},y_{i})$. By applying Hoeffding's inequality, we obtain 
	\begin{align}
	\mathcal{P}\left(\left\{ S:,|L_{S}(f)-L_{\mathcal{D}}(f)|>\epsilon\right\} \right)\leq2\exp(-2m\epsilon^{2}/(b-a)^{2}). \nonumber
	\end{align}
	We further apply union bound to yield 
	\begin{align*}
	& \mathcal{P}\left(\left\{ S:\exists f\in\mathcal{F},|L_{S}(f)-L_{\mathcal{D}}(f)|>\epsilon\right\} \right)\\
	& \leq\sum_{f\in\mathcal{F}}2\exp(-2m\epsilon^{2}/(b-a)^{2})\\
	& =2|\text{\ensuremath{\mathcal{F}|\exp(-2m\epsilon^{2}/(b-a)^{2})}}. \nonumber
	\end{align*}
	Therefore, we prove Lemma \ref{lemma1} with $\delta=2|\mathcal{F}|\exp(-2m\epsilon^{2}/(b-a)^2)$.
\end{proof}
Note that in Lemma~\ref{lemma1}, we assume $\mathcal{F}$ is finite. For an infinite $\mathcal{F}$, 
we can obtain a similar error bound using Rademacher complexity~\cite{Shai_2014_understandingML}. 

In Lemma~\ref{lemma1}, for a fixed error rate $\epsilon$, the probability $1-\delta$ increases with the amount of data $m$. 
This indicates that the learning algorithm with more data could return a model that can achieve the desirable error rate of $\epsilon$ with a higher probability. 
Since an agent that participates in federated learning can obtain a model that learns data from other agents, the agent could obtain the model with less generalization error than the model that is learned only on its local data. Thus, an agent that participates in federated learning has the gain.

On the other hand, in HFFL the higher-level model that can learn from more samples can have a less generalization error with a higher confidence according to Lemma~\ref{lemma1}.
Thus, our proposed algorithms (HFFL and HFFL$+$) where a higher level agent contributing more data could receive more reward (i.e., a model with less generalization error) aligns with our proportional fairness notion. 

\section{Experiments}

\label{section:experiments}
We conduct experiments with four datasets to validate HFFL and HFFL$+$ aligning with the fairness notion: More contribution levels has more rewards. 
%To study the HFFL and HFFL$+$ algorithms and validate their compliance with our fairness notion, 
\begin{table}[h!]
	\centering
	\begin{tabular}{lrrr}  
		\toprule
		Dataset  & $m_1$ ($l=1$)  & $m_2$ ($l=2$)  & $m_3$ ($l=3$) \\
		\midrule
		ADULT       & 200  & 500  &   2,000  \\
		MNIST    & 200  & 500 &  2,000     \\
		F-MNIST   &   400 & 1000 &  4000    \\
		IMDB            & 256  & 512   & 3584  \\
		\bottomrule
	\end{tabular}
	\caption{The amount of training data across agents at different levels. An agent at level $l$ randomly samples $m_l$ amount of data without replacement. }
	\label{tab:exp_training_data_allocations}
\end{table}

\begin{table}[h!]
	\centering
	\begin{tabular}{lrrr}  
		\toprule
		Dataset  & model $\mathcal{F}_1$ (red)  & model $\mathcal{F}_2$ (blue) & LR \\
		\midrule
		ADULT       & 2-layer MLP  & Logistic regression   &   0.01   \\
		MNIST    & 3-layer MLP  & 4-layer CNN &  0.01     \\
		F-MNIST   & 4-layer MLP & 4-layer CNN &  0.003    \\
		IMDB            & Bi-LSTM  & LSTM   &  0.001  \\
		\bottomrule
	\end{tabular}

	\caption{The different models used in HFFL for different datasets. HFFL uses the same model across agents at different levels for the same dataset.  LR is the learning rate (LR) of ADAM optimizer. }
	\label{tab:exp_settings}

\end{table}

\textbf{ADULT} is a census dataset from the UCI Machine Learning Repository~\cite{Dua:2019}. It is for the binary classification task predicting whether the personal income exceeds $\$50,000$ based on 14 attributes such as age, occupation, native country and so on. There are $31,561$ records of training data and $16,281$ records of test data. 
%In this experiment, 200 data is assigned to each agent in level 1, 500 data is assigned to each agent in level 2, and $2,000$ data is assigned to each agent in level 3.
%In the left-most panel in Figure~\ref{fig:HFFL}, we use HFFL to collaboratively train logistical regression models (LR, blue line) and neural networks of two fully connected layers (MLP, red line) for  agents at different levels. 
%As for HFFL$+$, HFFL$+$ calls the algorithm HFFL multiple times with different types of models. The candidate set of models is $\{$LR, MLP$\}$. 
%All agents use ADAM optimizer~\cite{Jimmy_Ba_ADAM} with the fixed learning rate of 0.01. 

\textbf{MNIST}~\cite{lecun2010mnist} is handwritten digits dataset that has a training set of $60,000$ examples, and a test set of $10,000$ examples. It is for the classification task of recognizing 10 digits from 0 to 9. %In this experiment, 200 examples are assigned to each agent in level 1, 500 examples are assigned to each agent in level 2, and $2,000$ examples are assigned to each agent in level 3. 
%In the second left panel of Figure~\ref{fig:HFFL}, for HFFL we use neural networks with three fully connected layers (MLP, red line) and convolutional neural network (CNN) with two convolutional layers followed by two fully connected layers (blue line). 
%For HFFL$+$ the candidate set is $\{$CNN, MLP$\}$. All agents use ADAM optimizer with fixed learning rate of 0.01. 

\textbf{Fashion MNIST} (F-MNIST) dataset~\cite{xiao2017_fashion_mnist} is an MNIST-like dataset with clothing images of 10 classes instead of handwritten digits. It consists of a training set of 60,000 examples and a test set of 10,000 examples. %Each data is 28x28 grayscale image, associated with a label from 10 classes.
%In this setting, 400 examples are assigned to each agent in level 1, $1,000$ examples are assigned to each agent in level 2 and $4,000$ examples are assigned to each agent in level 3. In the second right panel of Figure~\ref{fig:HFFL}, for HFFL we use the dropped-out MLP with 4 fully connected layers (red line) and the CNN with two convolutional layers followed by two fully connected layers (blue line). For HFFL$+$ the candidate set is $\{$CNN, MLP$\}$. All agents use ADAM optimizer with fixed learning rate of 0.003. 

\textbf{IMDB} dataset~\cite{maas-EtAl:2011:ACL-HLT2011_imdb} has 50K movie reviews for natural language processing for text analytics. It is a dataset for binary sentiment classification split into a training set of $25,000$ movie reviews and a test set of $25,000$. 
%256 examples are assigned to each agent in level 1, 512 examples are assigned to each agent in level 2 and $3,584$ examples are assigned to each agent in level 3. In the right-most panel of Figure~\ref{fig:HFFL}, our HFFL uses two-layer LSTM (blue line) and two-layer bidirectional LSTM (Bi-LSTM, red line). The candidate set of HFFL$+$ is $\{$LSTM, Bi-LSTM$\}$. All agents use ADAM optimizer with fixed learning rate of 0.001.

In all the experiments, we have three hierarchical levels ($L=3$), with level 1 having 20 agents, level 2 having 10 agents and level 3 having 4 agents, i.e., $N_1 = 20$, $N_2 = 10$ and $N_3=4$. 
Each agent has its local data randomly drawn from training data without replacement (the training data allocation refers to Table~\ref{tab:exp_training_data_allocations}).  

In all experiments of HFFL, we set the number of communication rounds $T=10$.
%At each level of federated learning in our HFFL algorithm, we choose FedAvg~\cite{McMahanMRHA17_aistat_fed_learning}.  
%At level $n$, after receiving model parameters from the central server, participating agents train an epoch over their local data $m_n$ and then send their parameter update to the central server. The central server averages all parameters of participating agents, update the global model and then send new parameters to all participating agents. This communication between agents and the central server iterates 10 times at each level (i.e., $T=10$). 
%Agents in different levels are rewarded with models with different qualities. 
In Figure~\ref{fig:HFFL}, the model score is measured by test accuracy on test data. For HFFL we choose different model architectures (refer to Table~\ref{tab:exp_settings}). For HFFL $+$, only the best-scoring model is returned for each level.    
We compare models' qualities (measured by model score) across three levels obtained by HFFL and HFFL$+$.

\subsection{Experimental Results and  Analysis}
In all panels of Figure~\ref{fig:HFFL}, for HFFL (blue and red lines), a model with a higher score is returned for a higher contribution level. This is expected since high-level agents can exploit all data from lower-level agents while the low-level agents can gain from a limited amount of data (equal to what they have) from higher-level agents. The experimental results validate the fairness notion: More contribution leads to more reward. 

\begin{figure*}[tp!]
	\centering
	\includegraphics[scale=0.15]{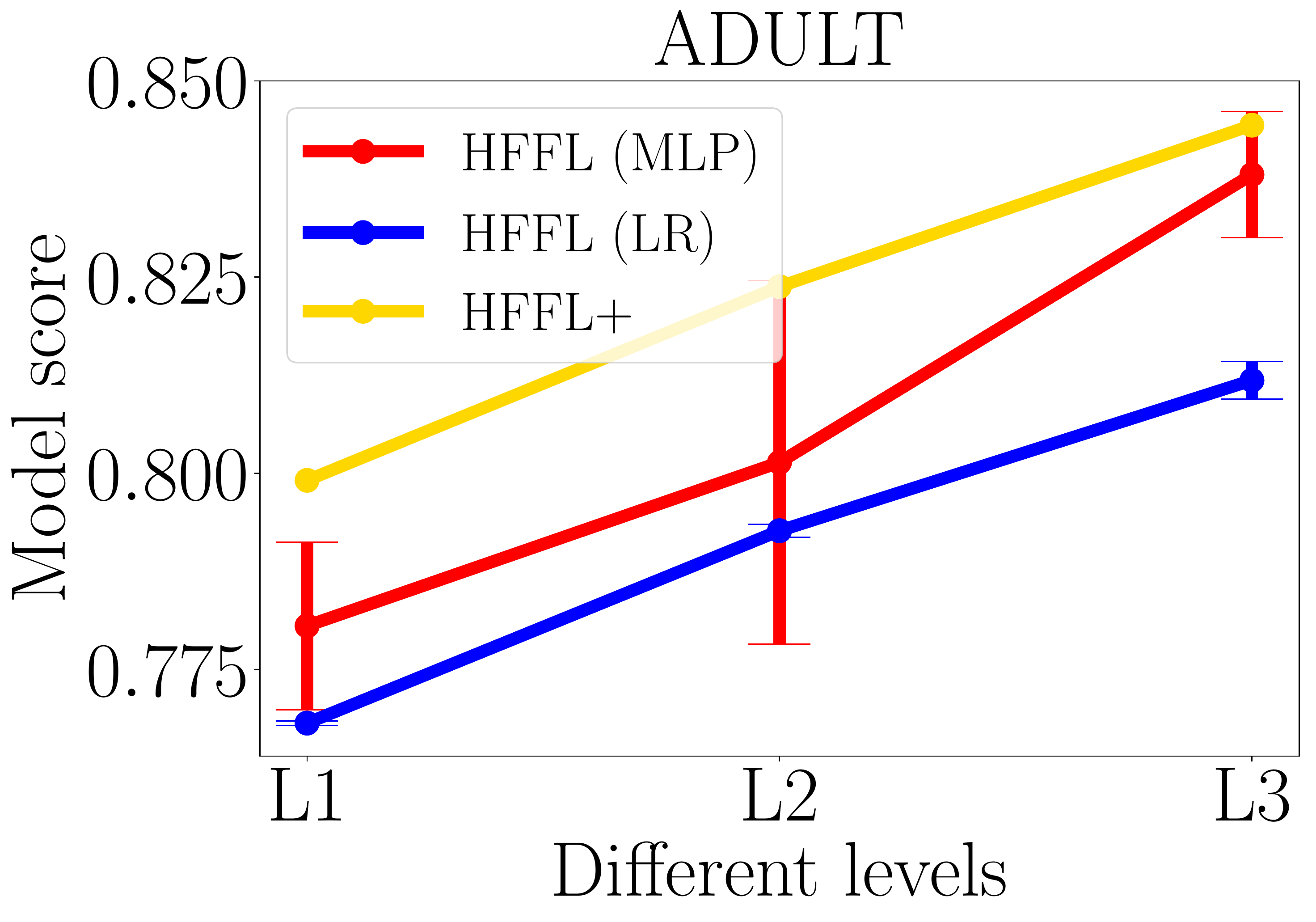}
	\includegraphics[scale=0.15]{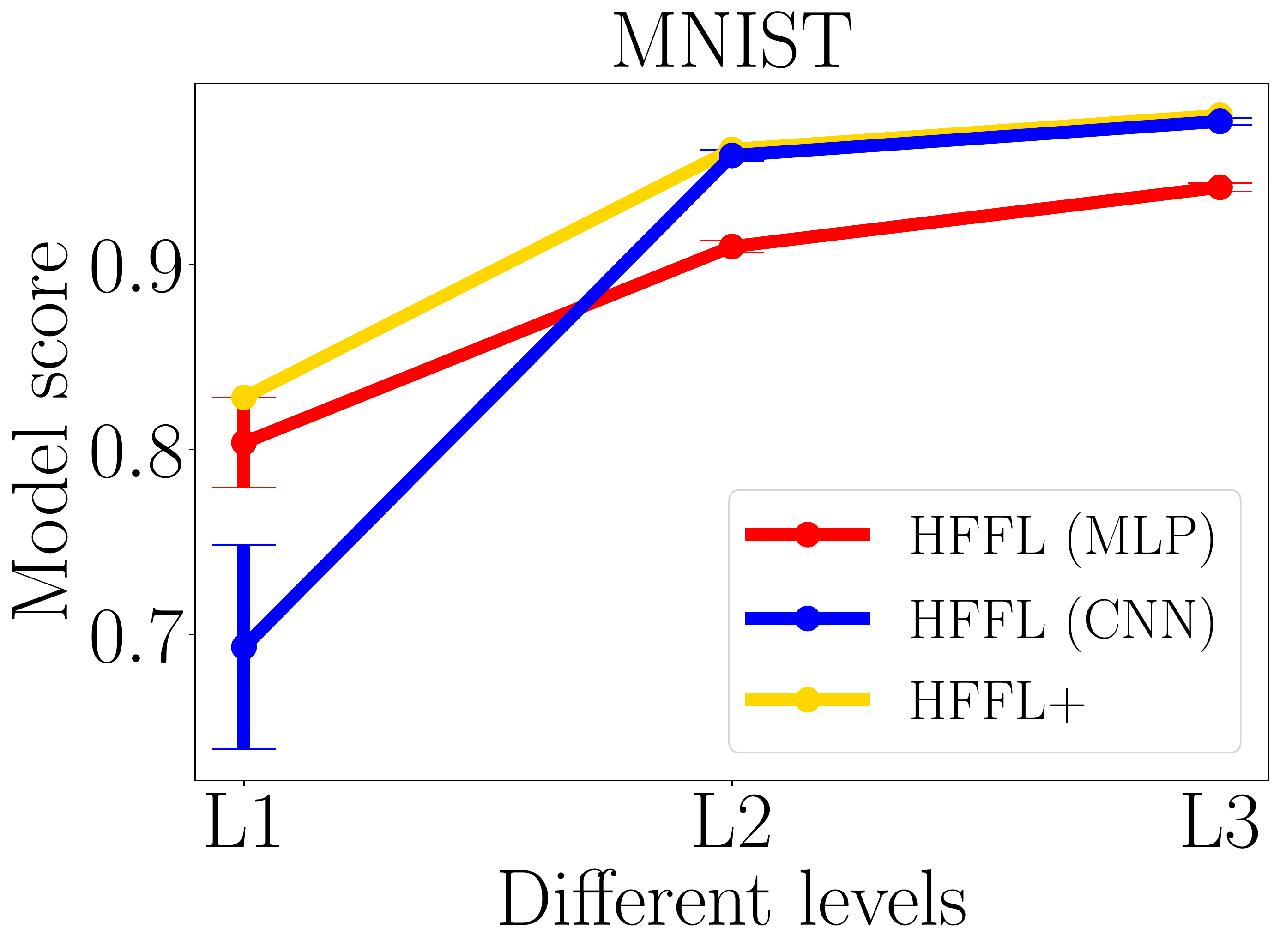} 
	\includegraphics[scale=0.15]{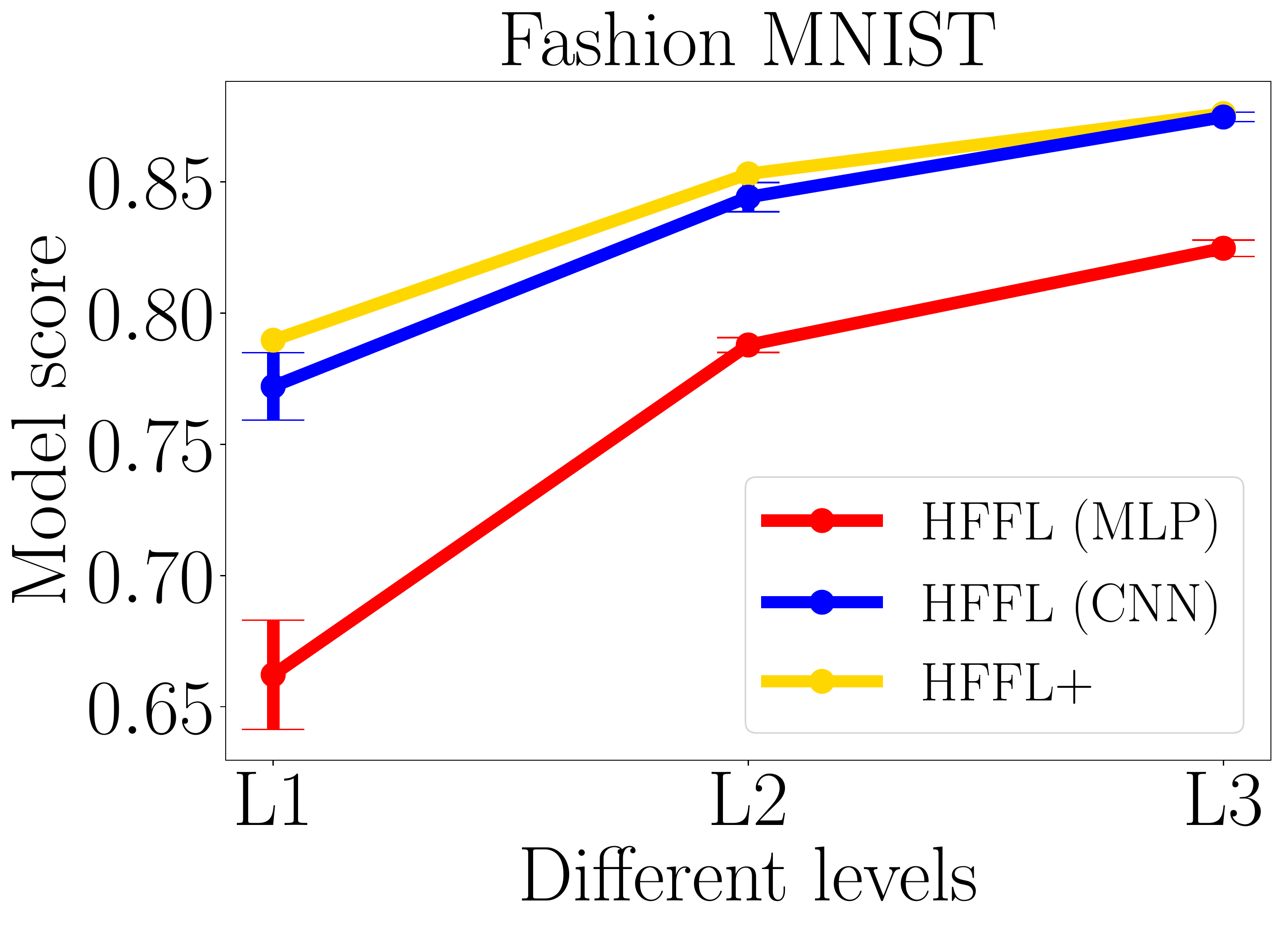}
	\includegraphics[scale=0.15]{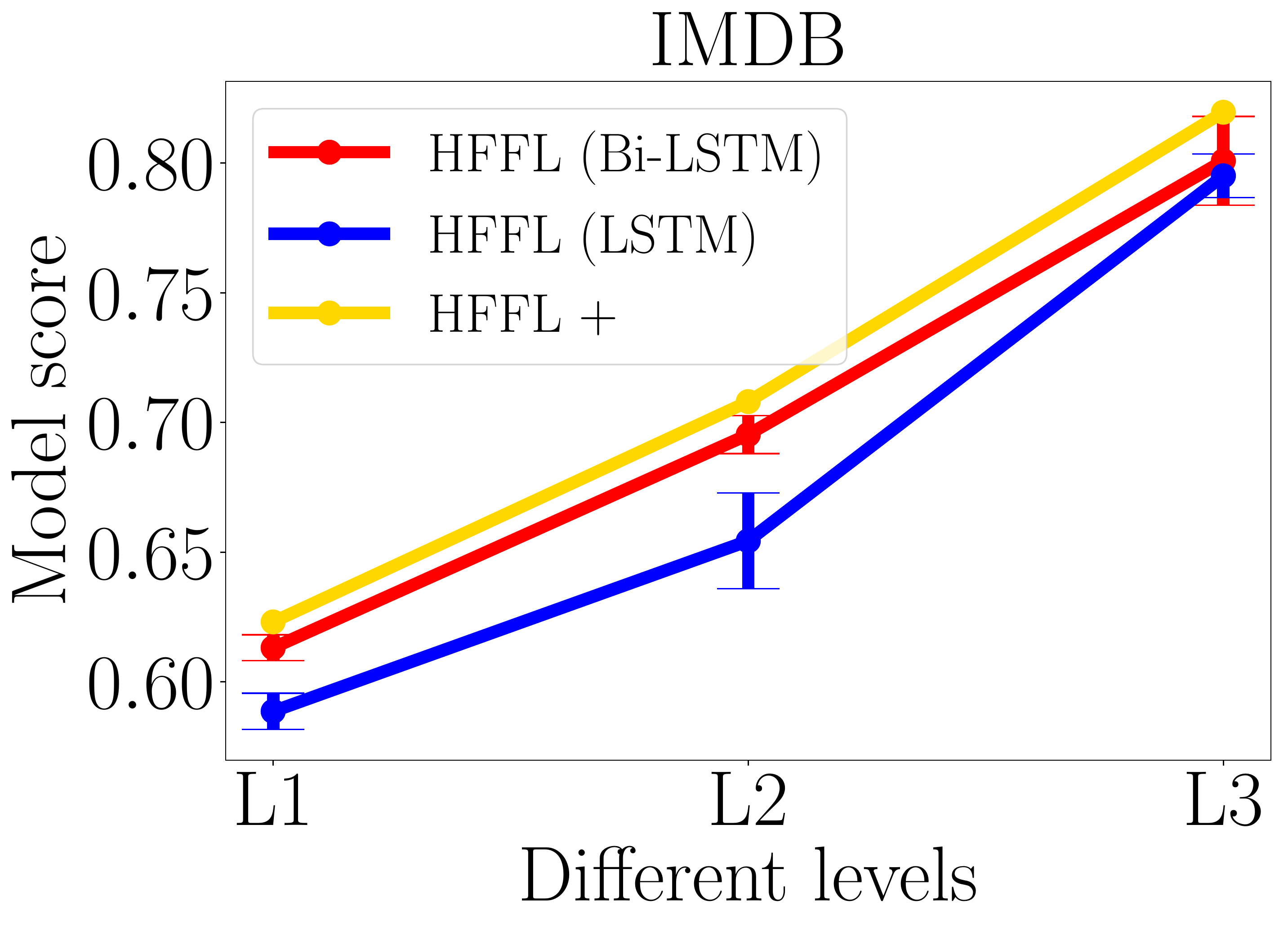}
	\caption{The model score of different levels in  HFFL and HFFL$+$ on four datasets: ADULT, MNIST, Fashion MNIST and IMDB. Model score is represented by the test accuracy on the test data. For HFFL, we run 5 trials with different random seeds and report median model score and its variance plotted as error bar. For HFFL$+$, we return the best-scoring model for each level. }
	\label{fig:HFFL}
\end{figure*}
%Our algorithms HFFL and HFFL$+$ learn different models for agents at different levels. Agents in higher levels contribute more data for federated learning. All panels in Figure~\ref{fig:HFFL} show agents at higher level could obtain higher scored models. This agrees to our fairness notion: more contribution, more reward.In addition, agents at a lower level (e.g., level 1) contribute less data, but also receive the model containing information of data that is from the agents at the same and higher levels. They contribute a certain amount of data in exchange for an equal amount of data from others. Thus, agents at a lower level still feel fair and will remain in collaboration. 

Our HFFL$+$ is the advanced version of HFFL. It runs HFFL multiple times employing models with different architectures. HFFL$+$ returns the best-scoring model for each level. 
As shown in all panels of Figure~\ref{fig:HFFL}, the yellow line (HFFL$+$) is always higher than red and blue lines (HFFL).
%Thus, the performance of HFFL$+$ is the upper bound of that of HFFL.
Thus, different from HFFL, HFFL$+$ could pick different types of model for different levels.  For example, in the MNIST experiment (second left panel in Figure~\ref{fig:HFFL}), agents at level 1 have less amount of data and therefore prefer a simpler model, e.g., MLP, and agents at higher levels (2 or 3) have more data and prefer a more complicated model that can benefit from having more data, e.g., CNN. 
HFFL$+$ returns the best-scutring MLP at level 1 and returns the best-scoring CNN at level 2 and level 3. 
%\vspace{-2mm}

\section{Conclusion and Future Work}
\label{section:conclusion}
%\vspace{-2mm}
In this paper, we propose a novel federated learning framework, HFFL, to achieve fairness among self-interested agents by rewarding agents hierachical model updates based on their contribution levels, thereby facilitating federated learning. 
We first identify agents' contributions based on publicly verifiable factors such as data quality, data volume, etc. Then, we develop the hierarchical federated learning framework, HFFL, that upholds the fairness notion. Experimental results indicate the efficacy of our proposed methods. 
%On the other hand, agents are encouraged to increase their contributions (e.g, data volume and data quality) in order to obtain more benefits. \\
%In this work we assume all agents are honest. 
Future potential research includes (a) how to handle dishonest or heterogeneous agents in HFFL, (b) how to make HFFL$+$ computationally efficient, e.g., by leveraging transfer learning, (c) comprehensively quantify all publicly verifiable data factors, and (d) introduce differential privacy into HFFL to thwart inference attacks on models.  
%For example, some adversarial agents send trash data in federated learning.\\
%Future work 2: To save computational time of HFFL$+$, we could use transfer learning. 
\clearpage

\bibliographystyle{unsrt}
\bibliography{bibfile}

\end{document}